\documentclass[11pt]{article}

\usepackage[preprint]{acl}

\usepackage{times}
\usepackage{latexsym}

\usepackage[T1]{fontenc}

\usepackage[utf8]{inputenc}
\usepackage{amsfonts, amsmath, amssymb, amsthm}
\usepackage{centernot}

\newtheorem{lemma}{Lemma}
\theoremstyle{definition} 

\usepackage{microtype}
\usepackage{algorithm}
\usepackage{algorithmic}
\usepackage{inconsolata}

\usepackage{graphicx}
\usepackage{subcaption} 
\usepackage{verbatim}
\usepackage{booktabs}
\title{Graph-Based Phonetic Error Correction of Noisy ASR}

\author{Pratik Rakesh Singh, Mohammadi Zaki, Aneesh Mukkamala \and Pankaj Wasnik \\
        Sony Research India\\
        \texttt{\{pratik.singh, mohammadi.zaki, aneesh.mukkamala, pankaj.wasnik\}@sony.com}}

\begin{document}
\maketitle
\begin{abstract}

Automatic speech recognition (ASR) systems, despite low overall word error rates, produce residual lexical errors that disproportionately affect semantically critical tokens such as named entities, negations, and sentiment-bearing words. These errors are often structured, arising from phonetic similarity rather than random noise, making naive token-level correction insufficient.
We propose a structured ASR correction framework, that we call G-SPIN,  that combines phonetic graph modeling with contextual language understanding. A graph neural network (GNN) first constructs acoustically plausible candidate neighborhoods for flagged tokens, explicitly restricting the correction search space to phonetic alternatives. A masked language model (MLM) then provides local contextual scoring, and an instruction-tuned large language model (LLM) performs final context-aware re-ranking over this compact candidate set. By decoupling structured phonetic reasoning from contextual semantic selection, our method avoids unconstrained generation while improving correction accuracy. The framework is lightweight, modular, and operates entirely at inference time.

\end{abstract}

\section{Introduction}

Automatic speech recognition (ASR) systems serve as a foundational interface between spoken language and text-based processing. Despite substantial progress and low aggregate word error rates, residual transcription errors remain inevitable in spontaneous, conversational, and acoustically challenging settings \cite{szymański2020werwerthink}. Importantly, these errors are not uniformly distributed across tokens, nor are they purely random \cite{leng2023softcorrecterrorcorrectionsoft}. Instead, ASR errors exhibit structured patterns, often governed by phonetic similarity, acoustic ambiguity, and contextual dependencies.

A central challenge in ASR correction lies in the nature of these perturbations. Many errors arise from systematic phoneme-level confusions—homophones, near-homophones, boundary shifts, or insertions and deletions driven by acoustic similarity \cite{wang2024dancer, nemoto-etal-2008-speech}. Consequently, effective correction cannot rely solely on generic contextual rewriting. Naïve token-level edits risk either failing to repair acoustically plausible confusions or introducing spurious replacements that are contextually fluent but phonetically implausible. \cite{Bouselmi2006MultilingualNS, chen-etal-2021-integrated}

Recent approaches have employed LLMs in prompt-driven settings to directly rewrite noisy transcripts. While flexible, such methods treat correction as unconstrained generation. As a result, they may hallucinate novel content, over-correct benign variations, or exhibit instability across similar contexts \cite{lyu2025dociaonlinedocumentlevelcontext, ma2025asrerrorcorrectionusing}. These behaviors reflect reliance on implicit in-context reasoning rather than explicit modeling of ASR noise structure. Recent work has also explored multi-stage ASR error correction using LLMs without explicit structural modeling. For example, the Reliable LLM Correction Framework (RLLM-CF) proposed by \cite{fang2025fewer} decomposes ASR correction into three LLM-driven stages—error pre-detection, iterative chain-of-thought subtask correction, and reasoning verification—to reduce hallucinations and improve correction accuracy. While this approach demonstrates promising reductions in character and word error rates, it remains fully dependent on prompt-based reasoning and iterative LLM calls for identification and rewriting, without explicit modeling of phonetic ambiguity or structured candidate generation. By contrast, our framework constrains the correction search space using a phonetic graph and integrates contextual scoring models, thereby reducing reliance on unconstrained LLM generation and improving correction consistency across contexts.

An alternative direction is to incorporate structured inductive bias into the correction process. Since many ASR errors arise from phonetic ambiguity, candidate corrections should be restricted to acoustically plausible alternatives before contextual reasoning is applied. This suggests a two-level correction paradigm: first constrain the lexical search space using phonetic structure, then apply contextual models to select the most coherent alternative.

In this work, we propose a structured ASR correction framework that explicitly separates phonetic reasoning from contextual selection. We construct a phoneme-level graph and train a graph neural network (GNN) to model relationships among acoustically similar lexical candidates. For each flagged token, the GNN defines a compact, structured candidate neighborhood. A masked language model (MLM) provides local contextual scoring, and an instruction-tuned LLM performs final context-aware re-ranking strictly within this constrained set. By reducing the correction search space prior to semantic reasoning, our approach avoids unconstrained rewriting while remaining flexible and context-sensitive.

Our framework is lightweight and modular, and primarily operates at inference time. While G-SPIN does not require any retraining or fine-tuning during deployment, it includes a one-time offline pretraining step for the GNN component, which captures phonetic relationships within a fixed vocabulary for a given language. This pretraining is data-independent and need not be repeated across datasets or domains. Once trained, the GNN is reused without modification, and the entire correction pipeline runs in an inference-only manner. In contrast to conventional approaches that require repeated fine-tuning for domain adaptation, our method avoids any task- or dataset-specific retraining.

Empirical results demonstrate that structured phonetic constraint combined with contextual re-ranking significantly improves correction accuracy over prompt-based rewriting and purely contextual baselines. These findings highlight the importance of integrating phonetic structure into modern ASR correction systems. Our contributions can be summarized as follows:

\begin{itemize}

\item 
We propose a principled ASR correction framework that combines phonetic structure and contextual understanding. A phoneme-level GNN constructs acoustically plausible candidate neighborhoods, and MLM scoring with instruction-tuned LLM re-ranking selects corrections strictly within this set, avoiding unconstrained rewriting and hallucinations.

\item 
Our method operates entirely at inference time without retraining ASR systems or LLMs, making it easily deployable and backbone-agnostic.

\item 
Extensive experiments on English, Telugu, Spanish, and Hindi demonstrate consistent improvements over strong LLM-based baselines for ASR correction across diverse linguistic settings.

\end{itemize}



\section{Problem Setup and Motivation}





\paragraph{Problem Setup.}

Let $x \in \mathcal{X}$ denote the intended (clean) sentence and 
$\tilde{x} = x + \delta$ its noisy ASR realization, where $\delta$ represents lexical perturbations arising from phonetic confusions, along with insertions, deletions, and boundary shifts. In this work, we focus primarily on correcting \emph{phonetic confusions}, which constitute a dominant class of ASR errors, by leveraging phonetic similarity structure. Concretely, we assume that erroneous tokens in $\tilde{x}$ can be mapped to a set of acoustically plausible alternatives defined over a phonetic neighborhood.

The ASR correction task seeks to recover a corrected sentence $\hat{x}$ from $\tilde{x}$ such that $\hat{x}$ is both (i) phonetically plausible under the acoustic evidence that produced $\tilde{x}$—operationalized via constraints over a phonetic similarity graph—and (ii) contextually coherent at the sentence level. Formally, we aim to design a correction operator
\(
C : \mathcal{X} \rightarrow \mathcal{X}
\)
that maps noisy transcripts to refined textual realizations:
\(
\hat{x} = C(\tilde{x}),
\)
with the objective that $\hat{x}$ closely approximates the latent clean sentence $x$.

\paragraph{Motivation.}
Automatic speech recognition (ASR) systems have achieved impressive reductions in aggregate word error rates. However, residual transcription errors remain inevitable in spontaneous, conversational, and acoustically challenging environments. Importantly, these errors are not uniformly distributed across tokens. Perturbations affecting semantically salient words—such as named entities, negations, or sentiment-bearing terms—can substantially alter the meaning of a sentence, even when overall error rates appear low.

Let $x$ denote the intended source sentence and $\tilde{x} = x + \delta$ its ASR realization. The perturbation $\delta$ is rarely arbitrary noise. Instead, ASR errors typically arise from structured phonetic confusions, including substitutions among acoustically similar words, insertions, deletions, or boundary shifts. As a result, many erroneous tokens are locally plausible yet globally inconsistent with sentence-level context.

Naive correction strategies that rely purely on contextual rewriting treat ASR repair as unconstrained generation. While such approaches may improve fluency, they risk hallucinating new content, over-correcting benign variations, or introducing inconsistencies across similar contexts. Conversely, purely string-based or edit-distance heuristics fail to capture long-range semantic dependencies and may under-correct impactful errors.

Effective ASR correction therefore requires integrating two complementary inductive biases: (i) phonetic structure, to restrict candidate replacements to acoustically plausible alternatives, and (ii) contextual language understanding, to select the most coherent option within that constrained set. The challenge is to design a lightweight inference-time mechanism that combines these signals without relying on unconstrained rewriting or expensive retraining.

Importantly, the need to restrict decoding to a structured phonetic space is not only empirically motivated but also theoretically grounded. As shown in Appendix~\ref{subsec:theoretical motivation} (Lemma~\ref{lemma:local_stability}), phonetic space restriction induces a contraction in the input perturbation norm, thereby improving the local stability of the frozen LLM under ASR noise. This formal result provides a principled justification for structured correction prior to downstream generation.

These considerations motivate a structured correction framework that explicitly models phonetic neighborhoods while leveraging contextual scoring for final selection, thereby balancing acoustic plausibility with semantic coherence.


\section{Methodology}
Here we discuss our Methodology Graph-based Structured Phonetic INference $(G-SPIN)$  the ASR input correction as a three-step process: 1) Faulty Input Identification, 2) Correct Input retrieval, 3) Input Correction. This three-step process ensures identifying noisy ASR input and replacing it with the best possible alternatives fitting the context. 

\paragraph{Step1: Faulty Input Identification: }
Automatic Speech Recognition (ASR) systems often produce noisy outputs, where acoustically confusable or out-of-vocabulary words are substituted with incorrect lexical forms. To identify such erroneous words without access to reference transcripts, we propose a Contextual Anomaly Detection (CAD) method based on a masked language model (MLM). 

Given a sentence $S = (w_1, \dots, w_n)$ embedded within its local discourse context $C$, we approximate the pseudo-log-likelihood (PLL) of each word by masking its constituent tokens and measuring the model’s confidence in reconstructing them. Let a word $w_i$ consist of tokens $\{t_{i1}, \dots, t_{im}\}$. For each token $t_{ij}$, we compute:\(P(t_{ij} \mid S_{\setminus t_{ij}}, C)\), where $S_{\setminus t_{ij}}$ denotes the input sequence with token $t_{ij}$ replaced by the \texttt{[MASK]} symbol. The word-level log-probability is then defined as the average token log-probability: \(
\log P(w_i \mid C) = \frac{1}{m} \sum_{j=1}^{m} \log P(t_{ij} \mid S_{\setminus t_{ij}}, C)\) and the corresponding word probability is:
\(
P(w_i \mid C) = \exp\left( \log P(w_i \mid C) \right).
\)

Additionally, we compute a minimum token confidence score:
\(
P_{\min}(w_i) = \min_j P(t_{ij} \mid S_{\setminus t_{ij}}, C),
\) which captures the weakest reconstruction confidence among the word’s sub-word tokens. To detect anomalous words, we first filter candidates whose word probability falls below a predefined threshold $\tau$:
\begin{equation}
P(w_i \mid C) < \tau.
\end{equation}
For the remaining candidates, we compute a combined anomaly score:
\[
\mathcal{A}(w_i) = \log P(w_i \mid C) 
+ \alpha \log \left( P_{\min}(w_i) \right),
\]
where $\alpha$ controls the contribution of the minimum token confidence. Words are ranked by $\mathcal{A}(w_i)$ in ascending order, and the top-$f$ (5 in our case) lowest-scoring words are flagged as potential ASR errors our experiment we use $\alpha = 0.5, \tau = 10^{-2}$ which are determined through experimentally.

This formulation enables context-aware detection of improbable lexical realizations without requiring supervision, making it suitable for large-scale post-processing of ASR outputs in realistic conversational settings.


\paragraph{Step2: Correct input retrieval.}
Previous studies observe that automatic speech recognition (ASR) frequently substitutes phonetically similar yet semantically unrelated tokens \cite{ruan2020asr,huang2020learningasrrobustcontextualizedembeddings}. To exploit this structure we convert the language model vocabulary into a phoneme-level graphical space and retrieve candidate correct inputs via graph-based similarity. Formally, let $\mathcal{V}$ denote the (LLM) vocabulary and let $\phi:\mathcal{V}\rightarrow \mathcal{P}^+$ be a phonemizer that maps each vocabulary item to a sequence of phonemes. We create a phoneme-node set
\[
\mathcal{N} = \big\{p \mid p \in \phi(v),\ v\in\mathcal{V}\big\},
\]
and obtain fixed-length phoneme embeddings via an embedding function $e:\mathcal{P}^+\rightarrow\mathbb{R}^d$ (e.g., averaging subtoken/phoneme encoder outputs). We then build an undirected graph $G=(\mathcal{N},\mathcal{E})$ where
\[
(p_i,p_j)\in\mathcal{E} \quad \text{iff} \quad \mathrm{cosine}\big(e(p_i),e(p_j)\big) \ge \eta,
\]
for a similarity threshold $\eta$ (0.9 in our case). Node attributes may include the phoneme embedding $e(p)$, enabling downstream disambiguation from phonetic signals.


\paragraph{GNN training.}
We train a GNN ($f_\theta$) for a link-prediction objective that encourages nodes representing confusable phonemes (and thus confusable words) to be strongly associated while separating spurious pairs. Let $h_v = f_\theta(v)$ denote the learned node representation. Define a differentiable pairwise score
\(
s(u,v) = h_u^\top h_v
\)
(or a small feed-forward scorer $s(u,v)=\mathrm{MLP}([h_u;h_v;h_u\odot h_v])$). For each positive pair $(u,v)\in\mathcal{P}^+$ we sample a set of negative nodes $\mathcal{N}_u^-$ (negative sampling) and minimize the binary cross-entropy style objective
\begin{align*}
\begin{split}
\mathcal{L}(\theta) =& -\sum_{(u,v)\in\mathcal{P}^+}\log\sigma\big(s(u,v)\big)\\
& - \lambda \sum_{(u,n)\in\mathcal{N}^-}\log\big(1-\sigma\big(s(u,n)\big)\big),     
\end{split}   
\end{align*}
where $\sigma$ is the sigmoid function and $\lambda$ balances positive/negative terms. In practice we construct positives using multi-hop connectivity: nodes within $\le H$ hops are treated as positive pairs (to capture one-to-many phonetic neighborhoods), while negatives are sampled uniformly or by degree-aware sampling to avoid easy negatives.

\paragraph{Inference and correction.}
At inference time, for a faulty ASR token $w_{\mathrm{ASR}}$, we first compute its phoneme representation $p_{\mathrm{ASR}} = \phi(w_{\mathrm{ASR}})$. We then locate the corresponding node $u$ in the graph and compute link scores $s(u,v)$ to candidate nodes $v$. The top-$k$ phoneme candidates are finally mapped back to vocabulary items in $\mathcal{V}$ using an inverted lexicon to produce correction candidates.

\paragraph{Step 3: Beam-search decoding with MLM scoring.}
Given a flagged ASR token $w_{\text{ASR}}$, the GNN provides a set of candidate phoneme-aligned token segments $\mathcal{S} = \{\mathcal{S}_1, \dots, \mathcal{S}_K\}$, where each segment $\mathcal{S}_i$ contains multiple candidate subword tokens. Our objective is to reconstruct the most plausible lexical candidate by composing one token from each segment. Since the search space grows exponentially with the number of segments, we employ beam search to efficiently explore the candidate space.

At decoding step $t$, each beam represents a partial candidate composed of subword pieces $\mathbf{p}_{1:t} = (p_1,\dots,p_t)$, which are detokenized into a surface form $c_t$. Each candidate is scored using a weighted combination of contextual language model likelihood.
\begin{align*}
    \begin{split}
        &\mathrm{Score}(c_t)
=
\lambda_{\mathrm{lm}} \cdot \log P_{\mathrm{MLM}}(c_t \mid S)
\\&+
\lambda_{\mathrm{edit}} \cdot \big(-\mathrm{NED}(c_t, w_{\text{ASR}})\big)
+
\lambda_{\mathrm{freq}} \cdot \mathrm{Freq}(c_t),
    \end{split}
\end{align*}
where $\mathrm{NED}(\cdot,\cdot)$ denotes normalized edit distance, and $P_{\mathrm{MLM}}(c_t \mid S)$ is computed using a masked language model (XLM-RoBERTa) by masking the original token position and evaluating the log-probability of the candidate’s subword tokens. We set $\lambda_{\mathrm{lm}} = 1, \lambda_{\mathrm{edit}} = 2, \lambda_{\mathrm{freq}} = 0.5$.

Beam search maintains the top-$B$ candidates (where $B$ = 10 in our experiment) at each step:
\[
\mathcal{B}_{t}
=
\operatorname{TopB}
\left(
\bigcup_{\mathbf{p}\in\mathcal{B}_{t-1},\ s\in\mathcal{S}_t}
\{\mathbf{p}\oplus s\}
\right),
\]
where $\oplus$ denotes concatenation. After processing all segments, the final decoded candidate is selected as:
\(
\hat{w}
=
\arg\max_{c \in \mathcal{B}_K}
\mathrm{Score}(c).
\)
This beam-based formulation enables efficient exploration of the large combinatorial candidate space while jointly incorporating phonetic similarity, contextual plausibility, and lexical priors to recover the most likely correction.

\begin{figure*}
    \centering
    \includegraphics[width=0.9\linewidth]{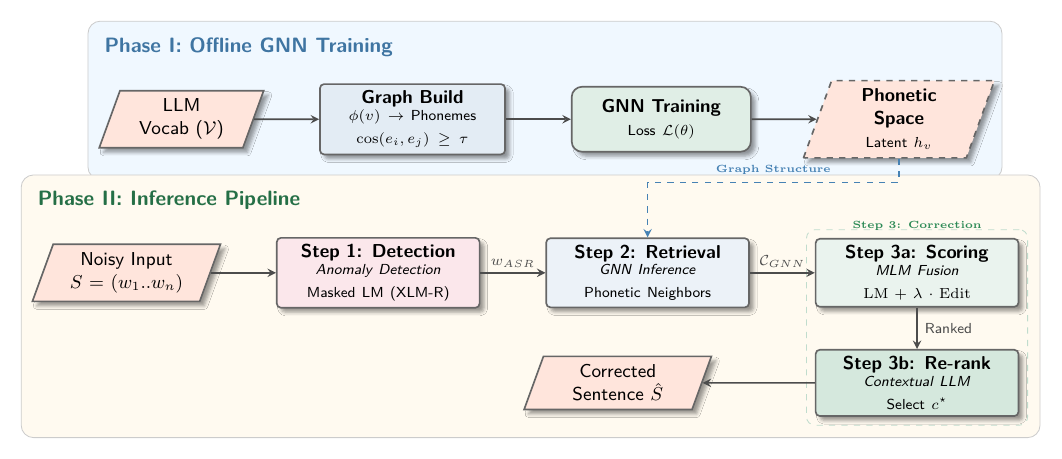}
    \caption{Block Diagram of the proposed Graph-Based ASR Correction strategy. Phase I represents the offline training phase of the GNN, while Phase II represents the inference pipeline.}
    \label{fig:block-diagram}
\end{figure*}




\paragraph{LLM-based contextual candidate selection.}
Although beam search with masked language model (MLM) scoring produces a set of plausible candidate corrections, the final selection requires deeper semantic and contextual reasoning beyond token-level likelihoods. To address this, we employ an instruction-tuned LLM (LLM), as a context-aware selection module. Given the noisy ASR sentence $S$, the flagged word $w_{\text{ASR}}$, and the decoded candidate set $\mathcal{C} = \{c_1,\dots,c_K\}$ from beam search, the LLM is prompted to select the single best replacement that fits the sentence context. The full prompt is shown in Appendix~\ref{subsec:prompt for llm decoding}.
Formally, the LLM defines a conditional selection function:
\[
\hat{c}
=
g_{\mathrm{LLM}}\left(S, w_{\text{ASR}}, \mathcal{C}\right),
\]
where $g_{\mathrm{LLM}}$ maps the noisy sentence, faulty token, and candidate set to the most contextually appropriate correction.


This LLM-based selection stage serves as a semantic correction layer on top of the phonetic graph and MLM-based decoding, enabling robust recovery of contextually appropriate lexical forms from noisy ASR outputs. An complete overview of our pipeline can be inferred through Figure~\ref{fig:block-diagram}.

\section{Experiments}
\subsection{Experimental Setup} We evaluate our method and baselines on Gemma-3-4b-it \cite{gemmateam2025gemma3technicalreport}, a strong and budget-friendly multilingual model. For faulty words detection (CAD) in Step 1 and Step 3 MLM-based decoding, we use XLM-RoBERTa for the MLM process \cite{conneau2020unsupervisedcrosslingualrepresentationlearning}. For training of GNN, we used GraphSAGE \cite{graphsage}, for link prediction, we use MLP layers connecting to the final layer of GNN. Below is the architectural setting for GNN in Appendix \ref{tab:gnn_settings}.

\subsection{Baselines} We compare our approach, G-SPIN, against several strong baselines for ASR error correction. First, we evaluate against DoCIA \cite{lyu2025dociaonlinedocumentlevelcontext}, a contextual LLM-based ASR refinement framework that leverages document-level context to improve transcription consistency and correctness. Second, we include RLLM-CF \cite{fang2025fewerhallucinationsverificationthreestage}, a three-stage LLM-based correction pipeline designed to reduce hallucinations and improve factual and contextual reliability in ASR refinement. In addition to LLM-based baselines, we also evaluate a graph-based baseline using the seed knowledge graph (KG) directly, without GNN retrieval. The baseline is a simple ASR noisy output. We evaluate G-SPIN and baselines on 3 metrics: WER (word error rate), SeMaScore $(S.)$ \cite{Sasindran2024SeMaScoreA}, and BertScore $(B.)$ \cite{zhang2020bertscoreevaluatingtextgeneration}. 
Furthermore, we evaluate GNN training on $Hits@10$ and $Hits@20$ and the AUC score.

\subsection{Dataset and languages} For our experimental setup, we use Loquacious-Set \cite{parcollet2025loquaciousset25000hours}, which contains noisy and clean output, making it perfect for our use case. 
To create parallel noisy–clean pairs, we systematically injected pure environmental and ambient noise samples into the clean audio. The noise types and mixing process were designed to approximate realistic deployment conditions encountered in real-world ASR scenarios. To ensure that the injected noise meaningfully impacted ASR performance without excessively degrading the linguistic content we designed a validation protocol across multiple noise levels. Specifically, we constructed a validation set spanning varying signal-to-noise ratios (SNRs) and evaluated transcription quality using both Word Error Rate (WER) and SeMa-Score. In addition to these quantitative metrics, we performed manual inspection to verify: The noise induced plausible ASR errors rather than arbitrary corruption, and the underlying semantic content of the utterances remained largely recoverable. We transcribe and translate audio for clean and noisy, we use seamless-m4t-v2-large\cite{seamless2023}. The languages we compare models are \texttt{hi}, \texttt{en}, \texttt{es}, and \texttt{te}.

\subsection{Results: }
Table~\ref{tab:main_tab} shows that G-SPIN outperforms all baselines in terms of WER, demonstrating the effectiveness of our GNN-based candidate generation combined with LLM-based contextual selection. G-SPIN also achieves the best SeMA score $(S.)$, although the margin over competing methods is relatively small, indicating comparable semantic preservation across approaches. In contrast, BERTScore $(B.)$ shows minimal variation among methods, suggesting that it is less sensitive to fine-grained lexical corrections and provides limited discriminative insight for ASR error correction performance.


\begin{table*}[t]
\centering
\small
\resizebox{\textwidth}{!}{%
\begin{tabular}{l|ccc|ccc|ccc|ccc}
\hline
\textbf{Methods} 
& \multicolumn{3}{c|}{\textbf{en}} 
& \multicolumn{3}{c|}{\textbf{es}} 
& \multicolumn{3}{c|}{\textbf{te}} 
& \multicolumn{3}{c}{\textbf{hi}} \\

& \textbf{WER} & \textbf{B.} & \textbf{S.}
& \textbf{WER} & \textbf{B.} & \textbf{S.}
& \textbf{WER} & \textbf{B.} & \textbf{S.}
& \textbf{WER} & \textbf{B.} & \textbf{S.} \\
\hline

ASR Baseline 
& 0.60 & 0.88& 0.64& 0.47 & 0.89& 0.72& 0.57 & 0.91& 0.80& 0.47 & 0.89& 0.78\\

DoCIA 
& 0.54 & 0.88& 0.71& 0.45 & 0.90& 0.75& 0.54 & 0.92& 0.82& 0.43 & 0.91& 0.80\\

RLLM-CF 
& 0.39 & 0.91& 0.80& 0.36 & 0.91& 0.81& 0.44 & 0.92& 0.87& 0.34 & 0.92& 0.85\\

KG 
& 0.37 & 0.91& 0.82& 0.35 & 0.91& 0.81& 0.42 & 0.92& 0.87& 0.36 & 0.92& 0.84\\

G-SPIN (ours)
& \textbf{0.32} & 0.91& \textbf{0.84}& \textbf{0.34} & 0.91& \textbf{0.83}& \textbf{0.39} & \textbf{0.93}& \textbf{0.89}& \textbf{0.31} & 0.92& \textbf{0.86}\\
\hline
\end{tabular}
}
\caption{WER $\downarrow$, BERTScore (B.) $\uparrow$ and SeMaScore (S.) $\uparrow$ comparison across languages. All scores lie between [0,1]}
\label{tab:main_tab}
\end{table*}

\begin{figure}[t]
    \centering
    \includegraphics[width=.9\linewidth]{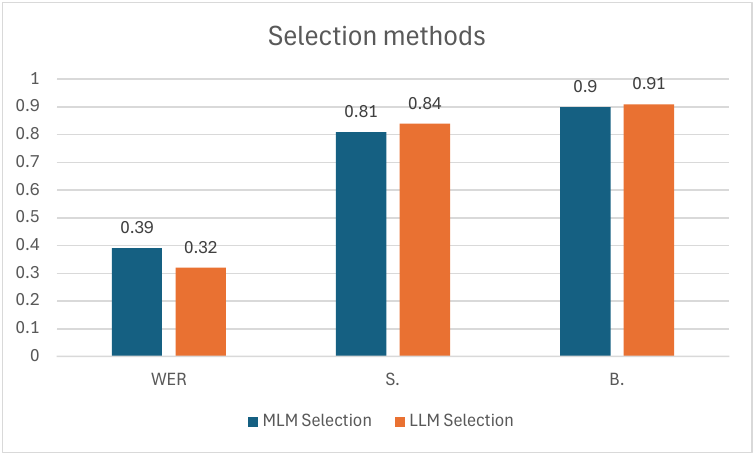}
    \caption{Comarison of different selection methods with the corresponding performance metrics.}
    \label{fig:scr}
\end{figure}

\begin{figure}[t]
    \centering
    \includegraphics[width=.9\linewidth]{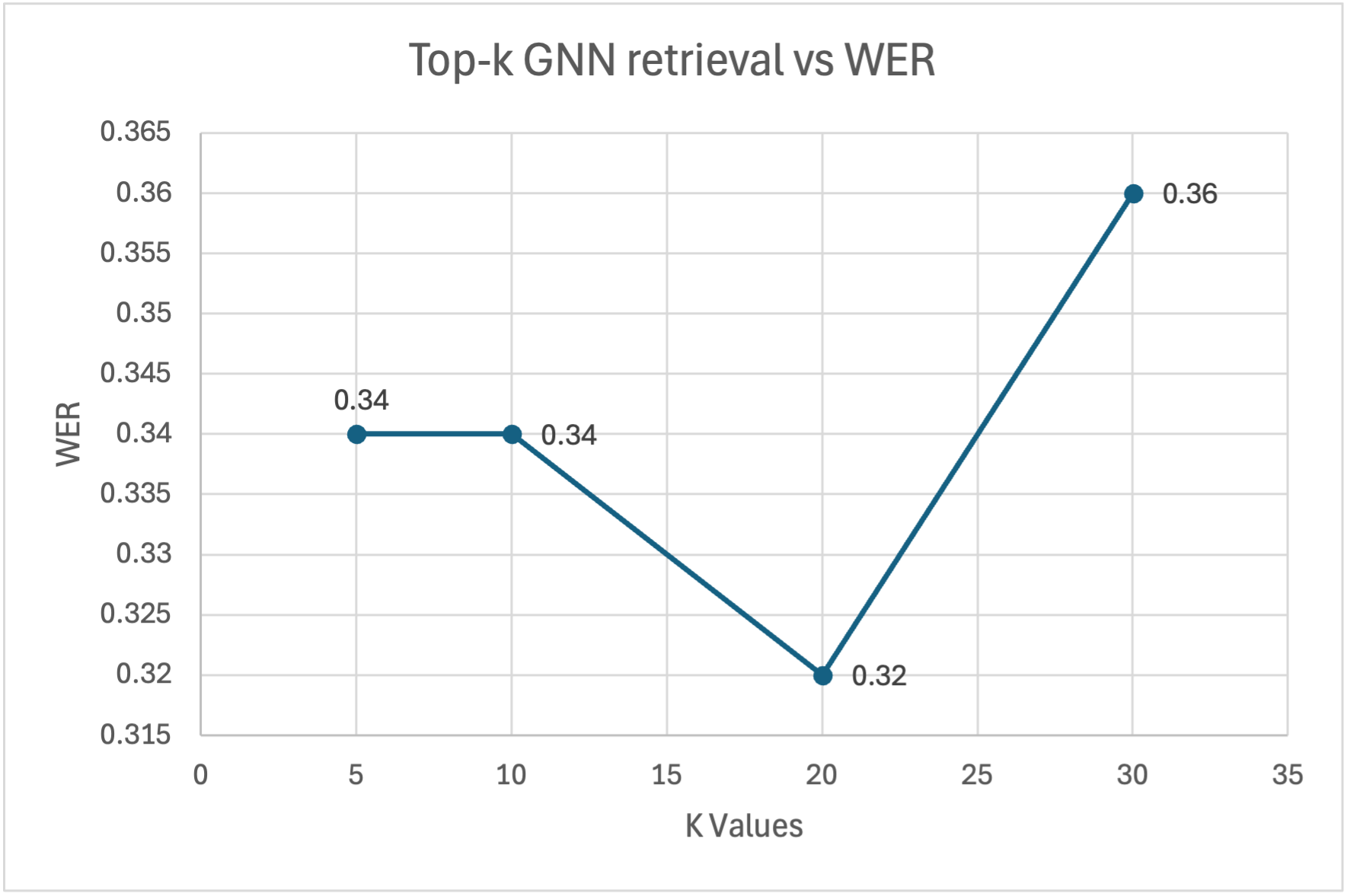}
    \caption{Ablation plot of value of $K$  vs WER.}
    \label{fig:rankwer}
\end{figure}

    

\subsection{Ablation Experiments}
\paragraph{Effect of Top-$K$ candidate selection.}
In Figure~\ref{fig:rankwer}, we analyze the impact of varying the number of Top-$K$ candidate tokens retrieved from the GNN during decoding. We observe that for smaller values such as $K=5$ and $K=10$, the WER remains largely unchanged, indicating that the highest-confidence candidates produced by the GNN are already highly relevant and sufficient for effective correction. Increasing the candidate pool to $K=20$ results in further improvement, suggesting that some optimal corrections may lie slightly beyond the top-ranked candidates and become accessible with a larger search space. However, when $K=30$, performance begins to degrade, likely due to the introduction of lower-quality or noisy candidates, which increase ambiguity during decoding and negatively affect candidate selection. 

\paragraph{MLM vs LLM based selection.} To evaluate contextual candidate selection, we compare our LLM-based selection against an MLM-based baseline, where the candidate with the highest masked language model probability is chosen. As shown in Figure~\ref{fig:scr}, the LLM-based method significantly outperforms MLM-based selection in terms of WER and achieves notable improvements in SeMA score, while both methods yield comparable BERTScore. These results indicate that LLM-based selection enables more effective context-aware disambiguation, leading to more accurate correction of ASR errors beyond token-level likelihood estimation.


\begin{table}[h]
\centering
\resizebox{\linewidth}{!}{%
\begin{tabular}{lccc}
\hline
\textbf{Languages} & \textbf{Hits@10} $\uparrow$ & \textbf{Hits@20} $\uparrow$ & \textbf{AUC} $\uparrow$ \\
\hline
en & 97.21 $\pm$ 0.65 & 99.02 $\pm$ 0.34 & 98.40 $\pm$ 0.30 \\
es & 98.71 $\pm$ 0.90 & 99.46 $\pm$ 0.68 & 98.90 $\pm$ 0.65 \\
hi & 95.57 $\pm$ 1.17 & 98.04 $\pm$ 0.93 & 96.66 $\pm$ 1.05 \\
te & 95.96 $\pm$ 1.00 & 99.39 $\pm$ 0.22 & 96.38 $\pm$ 0.48 \\
\hline
\end{tabular}
}
\caption{GNN training evaluation across languages. Higher is better for all metrics ($\uparrow$).}
\label{tab:gnn_train}
\end{table}

\paragraph{GNN link prediction performance.}
From Table~\ref{tab:gnn_train}, we observe that the GNN achieves strong performance across all languages on the link prediction task, as reflected by high Hits@10, Hits@20, and AUC scores. These results indicate that the model effectively captures phonetic similarity and reliably connects phonemes with acoustically and articulatorily similar counterparts. 
For Hindi and Telugu, the scores are slightly lower compared to other languages, which can be attributed to higher phonetic variability and increased ambiguity in phoneme–token mappings. This motivates the use of a larger candidate pool during retrieval. In particular, selecting Top-20 candidates from the GNN improves recall of correct phonetic matches, ensuring optimal correction candidates are included during downstream decoding and contextual re-ranking.

\subsection{Error Analysis and Qualitative Behavior}

We perform a detailed error analysis to better understand the strengths and limitations of our proposed method, G-SPIN, across different categories of ASR noise. For this analysis, we curate 200 samples for each error type from English audio and compute Word Error Rate (WER) across different noise categories (see Table~\ref{tab:error_analysis}).

\subsubsection{Substitution and Phonetically Similar Errors}

Our method is particularly effective at correcting substitution errors, especially those arising from phonetic ambiguity. This aligns with the design of G-SPIN, which explicitly models phonetic similarity during correction. As shown in Table~\ref{tab:error_analysis}, our approach significantly reduces errors in:
\begin{itemize}
    \item Grammatical errors: from 0.4037 $\rightarrow$ 0.1866
    \item Similar-sounding substitutions: from 0.4000 $\rightarrow$ 0.149
\end{itemize}

These improvements indicate that the proposed three-step pipeline—particularly the error detection and phonetic correction stages—is effective at identifying and resolving substitution-type errors. Qualitatively, we observe that G-SPIN is able to recover intended words even when the ASR output contains plausible but incorrect phonetic variants.

\subsubsection{Entity-Level Errors}

We also observe strong improvements in entity-related errors (e.g., named entities), with error rates reduced from 0.5107 $\rightarrow$ 0.192. This suggests that the model is capable of correcting semantically important tokens, which is critical for downstream real-world applications.

\subsubsection{Insertion Errors}

For insertion errors, G-SPIN achieves moderate improvements (0.4941 $\rightarrow$ 0.3656), but does not fully resolve them. This is expected, as the current framework is primarily designed around phonetic correction rather than sequence-level restructuring. However, we find that the faulty word detection stage (Step 1) is effective at identifying spurious inserted tokens that degrade sentence intelligibility. This suggests that, while G-SPIN alone does not completely eliminate insertion errors, it can serve as a strong precursor for downstream filtering or language model-based refinement.

\subsubsection{Deletion Errors}

Deletion errors remain the most challenging category, with only marginal improvement (0.73 $\rightarrow$ 0.71). This limitation stems from the inherent difficulty of recovering missing information from text-only ASR outputs without access to the underlying speech signal. To the best of our knowledge, accurately correcting deletion errors in such settings is fundamentally challenging, as the model must infer absent content without sufficient contextual or acoustic cues.







\begin{table}[t]
\centering
\resizebox{\linewidth}{!}{%
\small
\setlength{\tabcolsep}{4pt}
\renewcommand{\arraystretch}{0.95}
\begin{tabular}{lccccc}
\toprule
\textbf{Method} 
& \textbf{Gram.} 
& \textbf{Sim.} 
& \textbf{Ins.} 
& \textbf{Ent.} 
& \textbf{Del.} \\
\midrule
ASR 
& 0.404 & 0.400 & 0.494 & 0.511 & 0.73 \\
G-SPIN 
& \textbf{0.187} & \textbf{0.149} & \textbf{0.366} & \textbf{0.192} & \textbf{0.71} \\
\bottomrule
\end{tabular}}
\caption{Category-wise WER ($\downarrow$) on English ASR outputs (200 samples per category).}
\label{tab:error_analysis}
\end{table}

\section{Conclusion}

We introduce G-SPIN, a lightweight, inference-time framework that improves frozen LLMs under ASR noise. By combining phonetic constraints with graph-based contextual reasoning, it corrects transcription errors without retraining or unconstrained rewriting. Experiments on \texttt{en}, \texttt{es}, \texttt{hi}, and \texttt{te} show consistent gains, and theoretical analysis demonstrates that restricting the phonetic space improves local stability. G-SPIN offers a practical, plug-and-play approach to enhancing robustness in speech-driven applications.

\section*{Limitations}
One limitation of our work is that the effectiveness of G-SPIN depends on the quality of the phonetic neighborhoods constructed. If the candidate space fails to include the correct alternative, the correction module cannot recover the intended token. In highly noisy or domain-mismatched ASR settings, this may limit performance gains. G-SPIN’s performance degrades when the ASR output contains missing or dropped audio segments. In such cases, the method cannot recover the missing information, potentially resulting in incomplete or inaccurate corrections.

\section*{Ethics Statement}
Our work focuses on improving the robustness of multilingual ASR output correction to enhance accessibility and communication. However, the system relies on automatically generated transcriptions, which may contain errors or biases that could lead to misinterpretation, particularly in sensitive or minority-language contexts. Moreover, the datasets and underlying ASR systems may reflect demographic or linguistic biases, and while our method can correct transcription errors, it does not inherently remove these societal biases. Responsible deployment requires careful auditing for fairness across languages, accents, and dialects, as well as caution when using the system in critical applications.
\bibliography{acl_arxiv}

\appendix

\section{Theoretical Motivation}
\subsection{Local Stability under Structured Phonetic Projection}\label{subsec:theoretical motivation}

We provide a first-order analysis to understand why constraining corrections
to phonetic neighborhoods improves local stability.

\paragraph{Setup.}
Let $x \in \mathbb{R}^d$ denote the embedding representation of the clean sentence,
and let $\tilde{x} = x + \delta$ denote its ASR-corrupted version.
Let $f : \mathbb{R}^d \to \mathbb{R}$ be a continuously differentiable
contextual scoring function (e.g., masked likelihood, semantic consistency score).

Assume the correction operator $C$ produces:
\[
C(\tilde{x}) = x + \delta_C,
\]
where $\delta_C$ lies in a structured phonetic manifold
$\mathcal{M} \subseteq \mathbb{R}^d$,
representing acoustically plausible lexical directions.

We interpret $\delta_C$ as a projection of $\delta$ onto $\mathcal{M}$:
\[
\delta_C = P_{\mathcal{M}}(\delta).
\]

\begin{lemma}[Local Stability under Structured Projection]
\label{lemma:local_stability}
Let $f$ be continuously differentiable in a neighborhood of $x$.
Then for sufficiently small perturbations $\delta$,
\[
f(x+\delta) - f(x)
=
\nabla f(x)^\top \delta
+
\mathcal{O}(\|\delta\|^2).
\]

If the correction operator $C$ satisfies:
\[
\|\delta_C\| \le \alpha \|\delta\|,
\quad
\text{for some } \alpha < 1,
\]
then the first-order deviation obeys
\[
|f(x+\delta_C) - f(x)|
\le
\alpha \|\nabla f(x)\| \|\delta\|
+
\mathcal{O}(\|\delta\|^2).
\]
\end{lemma}

\begin{proof}
By first-order Taylor expansion,
\[
f(x+\delta)
=
f(x) + \nabla f(x)^\top \delta + \mathcal{O}(\|\delta\|^2).
\]
Applying this to $\delta_C$ gives:
\[
f(x+\delta_C)
=
f(x) + \nabla f(x)^\top \delta_C + \mathcal{O}(\|\delta_C\|^2).
\]
Using Cauchy–Schwarz,
\[
|\nabla f(x)^\top \delta_C|
\le
\|\nabla f(x)\| \|\delta_C\|.
\]
Substituting $\|\delta_C\| \le \alpha \|\delta\|$ completes the proof.
\end{proof}

\paragraph{Directional Attenuation.}
Beyond norm reduction, stability further improves when the projection
removes components aligned with high-sensitivity directions.
Let $g = \nabla f(x)$ and decompose:
\[
\delta = \delta_\parallel + \delta_\perp,
\]
where $\delta_\parallel$ lies in $\mathrm{span}(g)$.
If $\mathcal{M}$ excludes components aligned with $g$,
then
\[
|g^\top \delta_C|
<
|g^\top \delta|,
\]
even when $\|\delta_C\| \approx \|\delta\|$.

\paragraph{Interpretation.}
The correction operator therefore improves local stability via two mechanisms:
\begin{enumerate}
    \item \textbf{Norm contraction:} reducing perturbation magnitude.
    \item \textbf{Directional filtering:} attenuating components aligned with
    high-sensitivity semantic directions.
\end{enumerate}

Since ASR errors are primarily phonetic in nature,
constraining corrections to phonetic neighborhoods
acts as a structured projection that removes arbitrary semantic drift.
This yields improved first-order stability without requiring
retraining or modification of contextual language models.
\subsection{Prompt used by our LLM-based Decoding 
strategy}\label{subsec:prompt for llm decoding}

\begin{figure}[t]
    \centering
    \includegraphics[width=\linewidth]{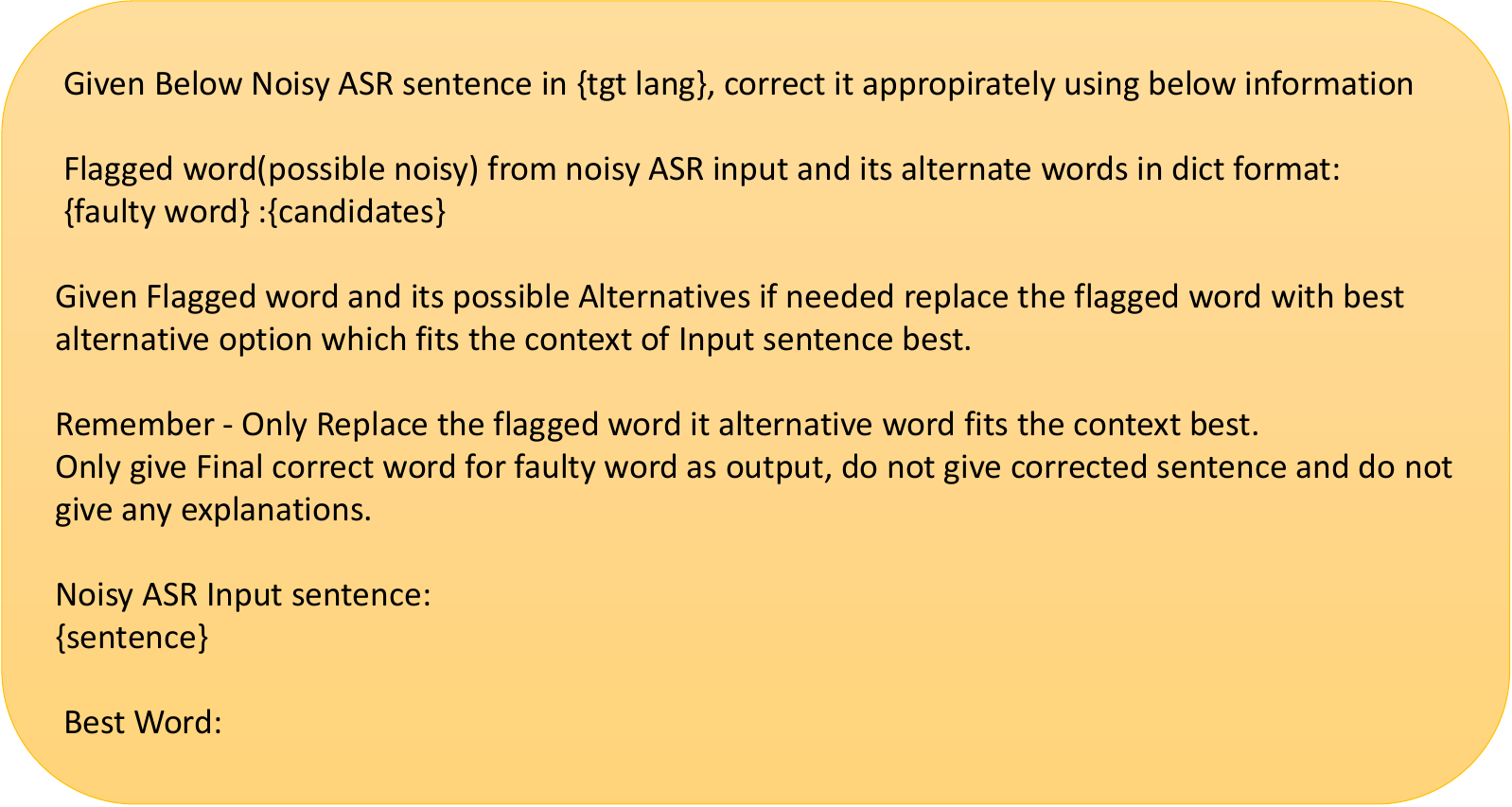}
    \caption{Prompt used for Ranking.}
    \label{fig:prompt for LLM}
\end{figure}

 
 
    
 
 


\paragraph{GNN training architecture details}
\begin{table}[t!]
\centering
\small
\begin{tabular}{ll}
\hline
\textbf{Component} & \textbf{Setting} \\
\hline
GNN Architecture & GraphSAGE (SAGEConv) \\
Input feature dimension & 768 \\
Hidden dimension & 256 \\
Number of GNN layers & 2 \\
Dropout rate & 0.5 \\
Activation function & ReLU \\
Link predictor & 2-layer MLP \\
Link predictor hidden dim & 256 \\
Link predictor output dim & 1 \\
\hline
\end{tabular}
\caption{Graph neural network architecture and training configuration used for phoneme-level link prediction.}
\label{tab:gnn_settings}
\end{table}

\subsection{Pseudo-code for ASR Error Correction.}
\begin{algorithm}
\caption{Graph-Guided ASR Error Correction (G-SPIN)}
\begin{algorithmic}[1]
\STATE \textbf{Input:} noisy sentence $S$, flagged tokens $\{w\}$, GNN dictionary
\FOR{each flagged token $w$ in $S$}
  \STATE $\mathcal{C}_{\text{GNN}} \leftarrow \text{CollectGNN}(w)$
  \STATE $\mathcal{C} \leftarrow \text{BeamDecodeMLM}(S, w, \mathcal{C}_{\text{GNN}})$
  \STATE $\mathcal{C} \leftarrow \text{Unique}(\mathcal{C} \cup \{w\})$
  \STATE $c^\star \leftarrow g_{\mathrm{LLM}}(S, w, \mathcal{C})$
  \IF{$c^\star$ is empty}
    \STATE $c^\star \leftarrow \arg\max_{c \in \mathcal{C}} \mathrm{Score}_{\mathrm{MLM}}(c)$
  \ENDIF
  \STATE Replace $w$ in $S$ with $c^\star$
\ENDFOR
\STATE \textbf{Output:} corrected sentence $S$
\end{algorithmic}
\end{algorithm}

\end{document}